\documentclass[11pt]{article}

\usepackage[T1]{fontenc}
\usepackage[utf8]{inputenc}
\usepackage{lmodern}
\usepackage{microtype}
\usepackage[a4paper,margin=28mm]{geometry}
\usepackage{amsmath,amssymb,amsthm,mathtools}
\usepackage{graphicx}
\usepackage{float}
\usepackage{booktabs}
\usepackage{array}
\usepackage{enumitem}
\usepackage{hyperref}
\usepackage{xurl}
\usepackage{needspace}
\hypersetup{
  colorlinks=true,
  linkcolor=black,
  citecolor=black,
  urlcolor=blue,
  pdftitle={DCCQ: From Ordered Bernoulli Levels to Critical-Line Geometry: Integer Quantization, Bernoulli Residual Phase, and Prime-Power Spectra},
  pdfauthor={Y. Kenan Yılmaz},
  pdfsubject={Discrete-complex complement quotient framework and zeta diagnostics},
  pdfkeywords={DCCQ, discrete-complex complement quotient, ordered Bernoulli levels, integer quantization, residual phase, prime-power spectra}
}

\newtheorem{proposition}{Proposition}[section]
\newtheorem{corollary}[proposition]{Corollary}
\theoremstyle{definition}
\newtheorem{definition}[proposition]{Definition}
\theoremstyle{remark}
\newtheorem{remark}[proposition]{Remark}

\newcommand{\Pp}{\mathbb{P}}
\newcommand{\Nn}{\mathbb{N}}
\newcommand{\Zz}{\mathbb{Z}}
\newcommand{\Cc}{\mathbb{C}}
\newcommand{\Rr}{\mathbb{R}}
\newcommand{\Log}{\operatorname{Log}}
\newcommand{\Arg}{\operatorname{Arg}}
\newcommand{\wtB}{\widetilde{B}_1}

\title{\textbf{DCCQ: From Ordered Bernoulli Levels to Critical-Line Geometry}\\
Integer Quantization, Bernoulli Residual Phase, and Prime-Power Spectra}
\author{Y. Kenan Y{\i}lmaz\\[0.35em]
\normalsize ORCID: \href{https://orcid.org/0009-0003-0368-556X}{0009-0003-0368-556X}}
\date{20 September 2026}

\begin{document}
\maketitle

\begin{abstract}
We present the discrete--complex complement quotient (DCCQ) framework through the ordered Bernoulli-word kernel $f(p,n,k)=p^k(1-p)^{n-k}$ and its inverse-integer level sets. The binary level $2^{-n}$ selects $p=1/2$ as the unique real split-independent anchor. Imposing the additional continuation rule $1-z=\overline z$ gives the conjugation-symmetric line $z=1/2+iu$. On the central normalized branch, the inverse-level split $\kappa=k/n$ also has real part $1/2$. The quadratic coordinate $Q(z)=1/4+u^2$ identifies complementary conjugate points, has minimum $1/4$, and has an explicit two-sign inverse; restricting its values to integers defines an arithmetic quantization. For supplied positive critical-line zero ordinates $\gamma_k$, the induced level has the exact decomposition
\[
L_k=\frac14+\gamma_k^2=N_k+\delta_k,\qquad
\delta_k=\widetilde B_1\!\left(\gamma_k^2+\frac34\right),
\]
where $N_k=\lfloor L_k+1/2\rfloor$ and $\widetilde B_1(x)=\{x\}-1/2$. The integer--phase pair $(N_k,e^{2\pi i\delta_k})$ retains the level and the positive ordinate. Prime-exponent coordinates describe the factorization of the integer component, while conditional Fourier moments formulate an open question about the residual phases. The finite numerical diagnostics are descriptive. A distinct complex exponent $s$ gives an inverse-level representation of the supplied term $m^{-s}$; direct substitution cancels the auxiliary coordinates. These terms assemble into the classical Dirichlet series and Euler product for $\Re s>1$. The framework organizes exact coordinate identities, classical zeta connections, and numerical diagnostics; it establishes neither conditional equidistribution nor a zero-location theorem.
\end{abstract}

\noindent\textbf{Keywords.} DCCQ; discrete--complex complement quotient; Riemann critical line; ordered Bernoulli likelihood; integer quantization; periodic Bernoulli function; residual phase; squared zero ordinates; Weyl statistics; von Mangoldt function; prime powers; prime factorization; Dirichlet series; Euler product; Bohr lift.

\clearpage
\section{Introduction}
Riemann's 1859 memoir placed the zeros of the analytically continued zeta function at the center of the relation between complex analysis and the distribution of primes \cite{Riemann1859,Titchmarsh1986}. The construction studied here approaches the familiar real-part-$1/2$ geometry from a different entry point. It begins not with $\zeta(s)$, but with the ordered Bernoulli-word kernel
\begin{equation}
 f(p,n,k)=p^k(1-p)^{n-k},
 \label{eq:kernel}
\end{equation}
and asks what geometry is obtained when inverse-integer levels are solved in a broad real/complex domain.

\paragraph{DCCQ terminology.}
We use \emph{discrete--complex complement quotient} (DCCQ) for the framework organizing these count, continuation, and symmetry coordinates. At the discrete starting point, a finite count pair $\mathbf c=(n,k)$ with positive integer $n$ and integer $0<k<n$ gives the rational ratio $k/n$. Identifying pairs with the same ratio forms a count quotient; dividing by $\gcd(n,k)$ selects its reduced representative. This operation retains the proportion and discards the common count multiplier. It does not identify the underlying ordered words or their likelihoods. Exchanging successes and failures sends $k/n$ to $1-k/n$ and exchanges the model parameters $p$ and $1-p$.

The complex level construction below extends the kernel and imposes a specified continuation rule. Its normalized split $\kappa=k/n$ is then an analytic coordinate, not an empirical proportion. On the chosen line, complementation becomes conjugation, and $Q$ identifies the two members of each complement pair. These are the complement and quotient roles used here. The continuous complex coordinates arise from the stated level equations and continuation choices; a finite rational count ratio is not thereby a general real or complex parameter.

The first level $2^{-n}$ uniquely selects $p=1/2$ as the real split-independent anchor. Preserving the complementary pair as complex conjugates gives
\[
z=\frac12+iu,\qquad 1-z=\frac12-iu,
\]
so a full conjugation-symmetric vertical line is obtained before a Dirichlet series, Euler product, or prime classification is introduced. A normalized split coordinate $\kappa=k/n$ inherits the same real part on the central phase branch.

The quadratic coordinate
\begin{equation}
Q(z)=z(1-z)=\frac14+u^2
\label{eq:Q}
\end{equation}
then supplies an exact real level geometry. Integer restriction of $Q$ produces a discrete family on the same vertical line. Known critical-line zeros may be passed through this coordinate diagnostically via $L_k=1/4+\gamma_k^2$. Each such real level admits an exact nearest-integer shell and centered residual, and the residual can be represented on the unit circle. This isolates $\gamma_k^2\bmod 1$ as the modular variable naturally associated with the quadratic coordinate.

A separate arithmetic layer uses the prime-exponent vector of an integer. Prime powers occupy one prime coordinate repeatedly, while integers with several distinct prime divisors occupy several coordinates. The scalar level lift depends only on the logarithmic resultant $\log m$, so it is a one-dimensional affine projection of the full factorization state. This is closely related in spirit to the classical Bohr multi-index viewpoint for Dirichlet series \cite{Hedenmalm1997}.

Finally, a distinct complex exponent $s$ is introduced in
\begin{equation}
F(z,s,k)=z^k(1-z)^{s-k}.
\label{eq:F}
\end{equation}
Solving $F(z,s,k)=m^{-s}$ realizes each Dirichlet atom inside the same multiplicative architecture. The common atom $m^{-s}$ then appears in the two standard zeta assemblies,
\[
\zeta(s)=\sum_{m\ge1}m^{-s}=\prod_{r\in\Pp}(1-r^{-s})^{-1},\qquad \Re s>1.
\]

The manuscript does \emph{not} claim a proof of the Riemann Hypothesis. The derived $1/2$ geometry belongs first to the Bernoulli coordinate $z$ and to the normalized split coordinate $\kappa$ under a specified conjugation-preserving continuation. The later Dirichlet exponent $s$ is introduced independently and is not assumed or proved to satisfy $\Re s=1/2$. The relationship with zeta zeros is therefore diagnostic and structural: exact coordinate identities are kept separate from classical prime-power/zero couplings, finite-sample numerical controls, and open conditional-equidistribution questions.

The principal exact statements are: (i) the binary level singles out the real anchor; (ii) the chosen complement-preserving continuation forces a conjugation-symmetric $1/2$ line; (iii) the quadratic coordinate has a sharp minimum and exact inverse; (iv) integer restriction yields arithmetic level quantization; (v) zero-induced levels have an exact periodic-Bernoulli residual; (vi) circularization yields a lossless integer/phase pair; (vii) prime-exponent space separates one-generator and multi-generator arithmetic states; and (viii) the complex-$s$ lift realizes the common Dirichlet atom that underlies both zeta assemblies.

\section{Ordered Bernoulli levels and the binary anchor}
Throughout, $\Nn=\{1,2,3,\ldots\}$ and $\Pp$ denotes the set of primes. For $n\in\Nn$, integer $0\leq k\leq n$, and $0<p<1$, the kernel \eqref{eq:kernel} is the likelihood of one specified ordered Bernoulli word containing $k$ successes and $n-k$ failures. No binomial coefficient is required because permutation classes are not being counted. Here $p$ is a model parameter; the observed ratio $k/n$ need not equal it. Real or complex exponents used below extend the kernel beyond finite count records.

\begin{proposition}[Unique split-independent binary anchor]
Let $0<p<1$ and $n>0$. The map
\[
k\longmapsto p^k(1-p)^{n-k}
\]
is independent of $k$ if and only if $p=1/2$. At this unique point its value is $2^{-n}$ for every finite real or complex $k$.
\end{proposition}
\begin{proof}
For positive real $p$,
\[
p^k(1-p)^{n-k}=(1-p)^n\exp\!\left(k\log\frac{p}{1-p}\right).
\]
Independence from $k$ is equivalent to $\log(p/(1-p))=0$, hence $p=1-p$ and $p=1/2$. Substitution gives $f(1/2,n,k)=2^{-n}$.
\end{proof}

\begin{proposition}[Centered solution of the binary level]
For $n>0$, at the centered split $k=n/2$, the real level equation
\[
p^{n/2}(1-p)^{n/2}=2^{-n},\qquad 0<p<1,
\]
has the unique solution $p=1/2$, with algebraic multiplicity two in the resulting quadratic equation.
\end{proposition}
\begin{proof}
Raising both sides to the power $2/n$ gives
\[
p(1-p)=\frac14,
\]
so
\[
\left(p-\frac12\right)^2=0.
\]
\end{proof}

\begin{remark}[Algebraic and arithmetic roles of 2]
The level $2^{-n}$ arises algebraically from the unique equal split of the Bernoulli complement pair. Separately, $2$ is the first prime and the unique even prime. The construction uses both facts, but does not identify them as logically equivalent.
\end{remark}

\Needspace{9\baselineskip}
\section{Complex phase and conjugation-symmetric continuation}
We first keep $p$ real and allow the exponent split to become complex.

\begin{proposition}[Vertical phase lattice]
Let $0<p<1$, $p\neq 1/2$, $n\in\Nn$, and $q>0$. Write $k=a+i\tau$. The level equation
\begin{equation}
p^k(1-p)^{n-k}=q^{-n}
\label{eq:realp-level}
\end{equation}
holds precisely when
\begin{align}
a&=n\,\frac{\log[1/(q(1-p))]}{\log[p/(1-p)]},\\
\tau&=\frac{2\pi j}{\log[p/(1-p)]},\qquad j\in\Zz.
\end{align}
Hence the solutions form a vertical arithmetic lattice in the complex $k$-plane.
\end{proposition}
\begin{proof}
Since $p$ and $1-p$ are positive real numbers,
\[
p^k(1-p)^{n-k}=p^a(1-p)^{n-a}\exp\!\left(i\tau\log\frac{p}{1-p}\right).
\]
Equality with the positive real number $q^{-n}$ requires equality of the magnitudes and phase closure modulo $2\pi$. Solving the two resulting equations gives the formulas above.
\end{proof}

\begin{remark}[A clean example]
For $q=3$ and $p=2/3$, the real part is $a=0$ and
\[
k_j=i\frac{2\pi j}{\log 2}.
\]
The spacing $2\pi/\log 2$ is a consequence of phase closure under complex exponentiation; it is not a statement about zeta-zero spacing.
\end{remark}

We next complexify the complementary Bernoulli pair by imposing the continuation rule
\begin{equation}
1-z=\overline z.
\label{eq:conjrule}
\end{equation}
Together with $z+(1-z)=1$, this yields
\[
z+\overline z=1,
\]
and therefore
\begin{equation}
z=\frac12+iu,\qquad 1-z=\frac12-iu,\qquad u\in\Rr.
\label{eq:zline}
\end{equation}
The real part $1/2$ is thus a consequence of the adopted complement-preserving continuation rule.

Write
\begin{equation}
z=\varrho e^{i\theta},\qquad 1-z=\varrho e^{-i\theta},\qquad
\varrho=\sqrt{\frac14+u^2},\qquad \theta=\arctan(2u).
\label{eq:polar}
\end{equation}
Because both factors lie in the open right half-plane, we use the principal logarithm
\[
\Log z=\log\varrho+i\theta,\qquad
\Log(1-z)=\log\varrho-i\theta.
\]

\begin{remark}[Derived statement versus continuation choice]
The real anchor $p=1/2$ is derived from the binary level together with split symmetry. Equation \eqref{eq:conjrule} is an additional complex-continuation rule. The statement $\Re z=1/2$ follows from that rule; the level equation alone does not force every complex parameter to lie on this line.
\end{remark}

\begin{proposition}[Integer-exponent level family]
Let $u\neq0$, $n\in\Nn$, and $q>0$. Under the principal logarithm, the solutions of
\[
z^k(1-z)^{n-k}=q^{-n}
\]
are
\begin{equation}
k_j(n;q,u)=\frac n2+\frac{\pi j}{\theta}+i\frac{n\log(q\varrho)}{2\theta},\qquad j\in\Zz.
\label{eq:kfamily}
\end{equation}
In particular, on the central phase branch $j=0$,
\[
\Re k_0=\frac n2.
\]
\end{proposition}
\begin{proof}
Define the powers using the principal logarithms of the two factors. By \eqref{eq:polar},
\[
z^k(1-z)^{n-k}
=\exp\!\left(n\log\varrho+i\theta(2k-n)\right).
\]
The exponent need not be the principal logarithm of the product. Equality of this exponential with $q^{-n}$ is equivalent to
\[
n\log\varrho+i\theta(2k-n)=-n\log q+2\pi i j,
\]
which rearranges to \eqref{eq:kfamily}.
\end{proof}

At $u=0$, the kernel is $2^{-n}$ for every $k$, so the integer-exponent level equation has every $k$ as a solution when $q=2$, and none when $q\ne2$. The formulas dividing by $\theta$ apply only to $u\ne0$.

After normalization $\kappa=k/n$, the central branch becomes
\begin{equation}
\kappa_q=\frac12+i\frac{\log(q\varrho)}{2\theta},\qquad q\ge2,
\label{eq:kappa}
\end{equation}
so
\[
\Re\kappa_q=\frac12.
\]
This normalized split is the DCCQ level coordinate in the $k$-plane: the common scale $n$ cancels on the central branch, while $u$ remains an analytic parameter of the chosen continuation.
For $q=2$, the continuous function
\[
t_2(u)=\frac{\log(2\varrho(u))}{2\theta(u)}
\]
extends by $t_2(0)=0$, is odd, and tends to $+\infty$ and $-\infty$ as $u\to+\infty$ and $u\to-\infty$, respectively. Indeed, $\varrho=1/(2\cos\theta)$, so $t_2=-\log(\cos\theta)/(2\theta)\sim\theta/4$ at zero. Continuity therefore gives the entire normalized vertical line. Its value at zero is a continuous branch choice within the degenerate solution family.

\section{DCCQ complement quotient and integer quantization}
Define
\[
Q(z)=z(1-z).
\]
On the conjugation-symmetric line \eqref{eq:zline},
\begin{equation}
Q(z)=\left(\frac12+iu\right)\left(\frac12-iu\right)=\frac14+u^2.
\label{eq:quadratic}
\end{equation}
Hence $Q(z)\ge1/4$, with equality if and only if $u=0$. The central point is therefore the unique minimum state:
\[
Q_{\min}=\frac14\quad\Longleftrightarrow\quad z=\frac12.
\]

Conversely, every real $L\ge1/4$ determines a conjugate pair
\begin{equation}
z_L^{\pm}=\frac12\pm i\sqrt{L-\frac14}.
\label{eq:inverseQ}
\end{equation}
Thus $Q$ is a coordinate on complement pairs: on this line, equal $Q$-values identify precisely $z$ and $1-z=\overline z$. The real level $L$ records the pair, and the sign in \eqref{eq:inverseQ} selects a member; the two members coincide at $L=1/4$. This is the quadratic complement quotient in the DCCQ framework.

Restricting $L$ to positive integers yields the exact discrete family
\begin{equation}
z_m^{\pm}=\frac12\pm i\sqrt{m-\frac14},\qquad m\in\Zz_{>0}.
\label{eq:integerquant}
\end{equation}
This is an arithmetic quantization of the continuous level coordinate. It does not, by itself, identify the resulting points with any externally specified spectral set.

For a nonunit integer level $m\ge2$, unique prime factorization separates three structural arithmetic forms:
\begin{center}
\begin{tabular}{lll}
\toprule
Level type & Arithmetic form & Generator interpretation\\
\midrule
Prime & $m=p$ & single prime generator\\
Proper prime power & $m=p^a$, $a\ge2$ & repeated single generator\\
Multi-prime composite & $m=\prod_j p_j^{a_j}$ & finite superposition of generators\\
\bottomrule
\end{tabular}
\end{center}
The integer $2$ may be isolated further because it also has the prior binary-anchor role.

\subsection{Zeta zeros as a diagnostic of the level coordinate}
Here the subscript $k$ indexes the supplied positive critical-line ordinates and is separate from the exponent split $k$ used in the kernel. For such a zero, write
\[
\rho_k=\frac12+i\gamma_k.
\]
The same quadratic map induces
\begin{equation}
L_k=Q(\rho_k)=\frac14+\gamma_k^2.
\label{eq:Lk}
\end{equation}
Using the independently computed ordinate cache described in Appendix~\ref{app:diagnostics}, the first positive ordinate $\gamma_1=14.134725\ldots$ gives
\[
L_1=200.0404548\ldots.
\]
The nearby integer $200=2^3\cdot5^2$ lies in the multi-generator composite class.

The nearest-integer diagnostic
\[
N_k=\left\lfloor L_k+\frac12\right\rfloor
\]
is a definition with a fixed tie convention; it does not assert $L_k=N_k$. For the first forty supplied ordinates, five rounded levels are prime. The observed count is $5/40=12.5\%$. As a first-order control, the prime number theorem suggests local prime density $1/\log x$ \cite{MontgomeryVaughan2007}. Summing over the same forty rounded levels gives
\[
\sum_{k=1}^{40}\frac{1}{\log N_k}\approx4.83.
\]
The observed count $5$ is numerically close to this heuristic sum. The prime number theorem does not supply a sampling model for these deterministic shell labels, so this comparison is descriptive and does not establish the presence or absence of prime enrichment.

\section{Exact integer--residual decomposition and Bernoulli phase}
The diagnostic rounding can be refined to an exact centered decomposition:
\begin{equation}
N_k=\left\lfloor L_k+\frac12\right\rfloor,
\qquad
\delta_k=L_k-N_k,
\qquad
-\frac12\le\delta_k<\frac12.
\label{eq:split}
\end{equation}
Hence
\[
L_k=N_k+\delta_k
\]
exactly, and $N_k$ has a unique factorization.
At a half-integer level $L_k=j+1/2$, this convention chooses $N_k=j+1$ and $\delta_k=-1/2$.

For the first supplied ordinate,
\[
L_1\approx200.0404548324,\qquad
N_1=200=2^3 5^2,\qquad
\delta_1\approx0.0404548324.
\]

Let $B_1(x)=x-1/2$ and define the periodic first Bernoulli function
\[
\wtB(x)=B_1(\{x\})=\{x\}-\frac12.
\]
Then the centered residual is exactly
\begin{equation}
\delta_k
=\left\{L_k+\frac12\right\}-\frac12
=\wtB\!\left(\gamma_k^2+\frac34\right).
\label{eq:bernoulliresidual}
\end{equation}
Consequently,
\begin{equation}
\frac14+\gamma_k^2
=\prod_j p_j^{a_j}
+\wtB\!\left(\gamma_k^2+\frac34\right),
\label{eq:fullsplit}
\end{equation}
where the product denotes the prime factorization of $N_k$.

The centered floor/$\wtB$ decomposition is universal for every real number. The special input here is the level $L_k=1/4+\gamma_k^2$, which is tied simultaneously to the quadratic Bernoulli coordinate and to a zeta-zero ordinate.

\subsection{Two distinct Bernoulli roles}
The construction uses two mathematically different Bernoulli structures. The local object $\wtB(\gamma_k^2+3/4)$ is the exact centered fractional residual of an individual level. The indexed Bernoulli numbers $B_2,B_4,B_6,\ldots$ are classical correction coefficients in Euler--Maclaurin and Gamma--zeta analysis \cite{Apostol1976,Titchmarsh1986}. For $n\ge2$,
\begin{equation}
B_n=-\frac{2n!\,\zeta(n)}{(2\pi)^n}\cos\!\left(\frac{\pi n}{2}\right),
\label{eq:Bn}
\end{equation}
and in particular
\[
B_{2m}=(-1)^{m+1}\frac{2(2m)!\,\zeta(2m)}{(2\pi)^{2m}}.
\]
The cosine supplies the parity/sign cycle; the remaining factor supplies the amplitude. No identity is asserted in which $\delta_k$ is a sum of higher Bernoulli corrections.

Euler--Maclaurin formulas provide a classical route to numerical zeta evaluation and hence, with root finding and error control, to approximate zero ordinates. The local residual is then computed from the supplied ordinate. The Bernoulli numbers alone do not select individual zeros.

\subsection{Prime-power spectrum and common-source interpretation}
The classical explicit-formula side uses the full von Mangoldt prime-power spectrum
\[
\Lambda(p^a)=\log p,\qquad a\ge1,
\]
and couples it globally to the nontrivial zeta zeros \cite{Titchmarsh1986,MontgomeryVaughan2007}. The explicit formula involves the full nontrivial zeros $\rho=\beta+i\gamma$, including their real parts. It cannot in general be replaced by a formula involving only the selected critical-line ordinates $\gamma_k$. The later map $\gamma_k\mapsto L_k\mapsto(N_k,\delta_k)$ applies to the supplied critical-line sequence. The global prime-power/zero coupling is not a one-to-one correspondence between individual prime powers and individual zeros.

While the nearest-integer cell is unchanged, differentiation gives
\begin{equation}
d\delta=dL=2\gamma\,d\gamma.
\label{eq:sensitivity}
\end{equation}
For $k=1$, the displacement of $\gamma_1$ from the integer shell is approximately
\[
\gamma_1-\sqrt{200-\frac14}\approx0.0014311,
\]
illustrating that the residual phase is more sensitive to zero-ordinate precision than the ordinate itself.

\subsection{Circular residual coordinate}
Because $\delta_k$ has a coordinate cut at $\pm1/2$, comparisons near that boundary are more naturally made on the unit circle. Define
\begin{equation}
Z_k=e^{2\pi i\delta_k}=e^{2\pi iL_k}=i\,e^{2\pi i\gamma_k^2}.
\label{eq:Zk}
\end{equation}
The integer shell disappears because $e^{2\pi iN_k}=1$. Thus the level has two complementary coordinates:
\[
\text{integer/arithmetic channel: }N_k=\prod_j p_j^{a_j},
\]
\[
\text{circular/residual channel: }Z_k=i\,e^{2\pi i\gamma_k^2}.
\]
Let $\Arg_{\mathrm c}$ denote the centered argument with values in $[-\pi,\pi)$, so $\Arg_{\mathrm c}(-1)=-\pi$. This convention matches the half-open residual interval; it is separate from the principal factor logarithms in Section~3. Then
\[
\delta_k=\frac{\Arg_{\mathrm c} Z_k}{2\pi},
\]
so $(N_k,Z_k)$ recovers $L_k=N_k+\delta_k$ and the positive ordinate $\gamma_k=\sqrt{N_k+\delta_k-1/4}$. A signed ordinate additionally requires its sign. Locally,
\[
\frac{dZ}{d\gamma}=4\pi i\gamma Z,
\qquad
d(\arg Z)=4\pi\gamma\,d\gamma.
\]
The modular core behind this channel is $\gamma_k^2\bmod1$.

\subsection{Conditional Weyl statistics}
The literature contains classical uses of Bernoulli numbers and periodic Bernoulli functions in Euler--Maclaurin/zeta analysis, together with results on distribution modulo one for certain sequences of zero ordinates \cite{Apostol1976,CicekGonek2024}. The present open question is not a deterministic formula $\delta_k=F(\text{factorization of }N_k)$, but whether arithmetic classes of $N_k$ leave a Fourier signature in $Z_k$.

For a positive integer $m$, $\omega(m)$ counts its distinct prime divisors and $\Omega(m)$ counts prime factors with multiplicity. For example, $200=2^3 5^2$ has $\omega(200)=2$ and $\Omega(200)=5$; both counts are zero at $m=1$.

For an arithmetic class $\mathcal A$, define
\[
\mathcal A_K=\{k\le K:N_k\in\mathcal A\}.
\]
For $|\mathcal A_K|>0$, the conditional circular moment is
\begin{equation}
\begin{aligned}
W_h(\mathcal A;K)
&=\frac{1}{|\mathcal A_K|}\sum_{k\in\mathcal A_K}Z_k^h\\
&=\frac{1}{|\mathcal A_K|}\sum_{k\in\mathcal A_K}e^{2\pi ih(\gamma_k^2+1/4)},
\qquad h=1,2,\ldots.
\end{aligned}
\label{eq:Weyl}
\end{equation}
Equivalently, replacing $1/4$ by $3/4$ multiplies the statistic by $(-1)^h$ and leaves its vanishing/nonvanishing content unchanged. This is a conditional version of exponential sums appearing in Weyl's criterion \cite{Weyl1916}. Candidate classes include prime shells, pure prime powers, semiprimes, fixed $\omega(N_k)$, fixed $\Omega(N_k)$, and richer factorization signatures. The proposed null is conditional uniformity: for every nonzero integer harmonic, $W_h(\mathcal A;K)\to0$ as $|\mathcal A_K|\to\infty$. Independence from the arithmetic class alone does not imply uniformity; a common phase distribution can be nonuniform. In a joint limiting model, uniform phase and independence suffice for a class of positive limiting frequency; sparse classes need separate control. A nonzero conditional harmonic alone therefore does not establish a class effect. Compare it with the unconditional moment $K^{-1}\sum_{k\le K}Z_k^h$.

The finite correlations with $\omega(N_k)$ and $\Omega(N_k)$ in Appendix~\ref{app:diagnostics} are descriptive and do not establish independence. The cited result of \c{C}i\c{c}ek and Gonek concerns derivative-selected ordinates under RH and a spacing hypothesis; it does not settle these arithmetically selected squared phases \cite{CicekGonek2024}. No additional unreported statistical control is used here.

\section{Prime-coordinate modes and scalar collapse}
The arithmetic role classification is most precise before scalar projection.

\begin{definition}[Prime-exponent state space]
Let $c_{00}(\Pp)$ denote the real vector space of finitely supported sequences indexed by the primes, with canonical basis $(\mathbf e_r)_{r\in\Pp}$. For a positive integer $m$ with nonnegative integer exponents $a_r$ of finite support,
\[
m=\prod_{r\in\Pp}r^{a_r},
\]
define
\begin{equation}
\mathbf a(m)=\sum_{r\in\Pp}a_r \mathbf e_r\in c_{00}(\Pp).
\label{eq:alpha}
\end{equation}
The logarithmic functional $\mathcal L:c_{00}(\Pp)\to\Rr$ is
\begin{equation}
\mathcal L\!\left(\sum_r x_r \mathbf e_r\right)=\sum_r x_r\log r.
\label{eq:logfunctional}
\end{equation}
For integer states, $\mathcal L(\mathbf a(m))=\log m$.
\end{definition}

On the central branch of the integer-exponent level family, define
\[
k_1=\frac n2+i\frac{n\log\varrho}{2\theta},
\qquad
\Delta_r=i\frac{n\log r}{2\theta}.
\]
Then
\begin{equation}
k_m=k_1+i\frac{n\log m}{2\theta}.
\label{eq:km}
\end{equation}

\begin{proposition}[Prime powers and multi-prime composites]
Let $n\in\Nn$, $u\ne0$, and $m=\prod_r r^{a_r}$ with nonnegative integer exponents of finite support.
\begin{enumerate}[label=(\alph*)]
\item If $m=r^a$ is a prime power, then
\[
\mathbf a(m)=a\mathbf e_r,\qquad k_{r^a}=k_1+a\Delta_r.
\]
Thus successive powers of one prime form an arithmetic trajectory with increment $\Delta_r$.
\item For a general integer,
\[
\mathbf a(m)=\sum_r a_r \mathbf e_r,\qquad
k_m=k_1+\sum_r a_r\Delta_r.
\]
Hence a multi-prime composite is a finite superposition of prime-coordinate occupations before scalar projection.
\end{enumerate}
\end{proposition}
\begin{proof}
Unique factorization gives $\log m=\sum_r a_r\log r$. Substitution into \eqref{eq:km} yields both statements.
\end{proof}

\begin{remark}[Mode versus eigenstate]
The word \emph{mode} is exact here in the coordinate sense: $(\mathbf e_r)$ is the canonical basis of prime-exponent space. The word \emph{eigenstate} would require an additional operator for which these basis vectors are eigenvectors; no such operator is assumed.
\end{remark}

\begin{proposition}[Affine-log collapse]
For fixed $n\in\Nn$ and $u\neq0$, the central-branch scalar level coordinate is the affine functional
\begin{equation}
k_m=k_1+B\,\mathcal L(\mathbf a(m)),\qquad B=i\frac{n}{2\theta}.
\label{eq:affine}
\end{equation}
Equivalently, $k_m=A+B\log m$ for constants $A,B$ independent of $m$.
\end{proposition}

\begin{corollary}[What is hidden and what is preserved]
The scalar map $\mathbf a(m)\mapsto k_m$ hides the separate prime-coordinate entries but does not destroy uniqueness of integer factorization. If $k_m=k_{m'}$ for positive integers $m,m'$ under fixed $n,u\neq0$, then $m=m'$ and therefore $\mathbf a(m)=\mathbf a(m')$.
\end{corollary}

Before projection, the prime factors occupy independent canonical coordinates. After projection, only the weighted sum $\sum_r a_r\log r$ is visible. Recovering the components from the integer $m$ is precisely the factorization problem. In a Bohr lift, the same integer is encoded by its prime multi-index $(a_r)$, and Dirichlet-series coefficients become Fourier or power-series coefficients indexed by those multi-indices \cite{Hedenmalm1997}. The present scalar level coordinate does not replace that richer space; it is a specific affine logarithmic projection arising from the Bernoulli level equation.

\section{Continuation to a complex spectral exponent}
We now replace the integer total exponent $n$ by a distinct complex variable $s\in\Cc$ and use \eqref{eq:F}. The half real part derived earlier belongs to the Bernoulli coordinate $z$ and normalized split coordinate $\kappa$; no restriction $\Re s=1/2$ is imposed. At this stage $F$ is a complex analytic extension of the multiplicative Bernoulli kernel, not a probability mass function.

\begin{proposition}[Lift of a Dirichlet atom]
Let $u\neq0$, $s\in\Cc$, and $m>0$. Under the principal-branch convention, the solutions of
\begin{equation}
F(z,s,k)=m^{-s}
\label{eq:atomlevel}
\end{equation}
are
\begin{equation}
k_j(s;m,u)=\frac s2+\frac{\pi j}{\theta}+i\frac{s\log(m\varrho)}{2\theta},\qquad j\in\Zz.
\label{eq:complexlift}
\end{equation}
The central branch is
\begin{equation}
k_m(s;u)=\frac s2+i\frac{s\log(m\varrho)}{2\theta}.
\label{eq:centralcomplex}
\end{equation}
\end{proposition}
\begin{proof}
With the principal logarithms of the factors, the continued kernel is
\[
F(z,s,k)=\exp\!\left(k\Log z+(s-k)\Log(1-z)\right)
        =\exp\!\left(s\log\varrho+i\theta(2k-s)\right).
\]
The displayed exponent need not be the principal logarithm of $F$. Equality with $m^{-s}=\exp(-s\log m)$ is equivalent to
\[
s\log\varrho+i\theta(2k-s)=-s\log m+2\pi i j,
\]
which gives \eqref{eq:complexlift}.
\end{proof}

For fixed $s\ne0$ and $u\ne0$, the central lift has the form $A+B\log m$ with $B\ne0$ and is injective on positive integer labels. At $s=0$, however, $k_m(0;u)=0$ for every $m$, so that injectivity is lost. Direct substitution gives
\[
F(z,s,k_m(s;u))
=\exp\!\left(s\log\varrho+i\theta(2k_m-s)\right)
=e^{-s\log m}=m^{-s}.
\]
Thus the auxiliary coordinates cancel. The label $m$ was supplied in the target level equation; the inverse split represents that prescribed atom and does not independently generate an integer or a zeta zero. The prime multi-index $\mathbf a(m)$ makes the factorization entries explicit.

\section{The two zeta faces from the common atom}
Define
\begin{equation}
K(m,s)=m^{-s}=e^{-s\log m}.
\label{eq:commonatom}
\end{equation}
The previous section realizes this kernel as a level value of $F$.

\subsection{Dirichlet face: enumeration}
For $\Re s>1$,
\begin{equation}
\sum_{m\ge1}F(z,s,k_m)=\sum_{m\ge1}m^{-s}=\zeta(s).
\label{eq:dirichletface}
\end{equation}
The inverse-level formula represents the supplied integer atoms; their global enumeration is the classical Dirichlet series.

\subsection{Euler face: prime-power generation}
The geometric identity
\[
(1-x)^{-1}=\sum_{a\ge0}x^a
\]
is the shape-one negative-binomial generating identity. For a prime $r$ and $\Re s>1$, setting $x=r^{-s}$ gives
\begin{equation}
(1-r^{-s})^{-1}=\sum_{a\ge0}r^{-as}=\sum_{a\ge0}K(r^a,s).
\label{eq:eulerfactor}
\end{equation}
Thus a local Euler factor generates the full prime-power tower $1,r,r^2,r^3,\ldots$ through the same atom $K$. Multiplication over primes gives
\begin{equation}
\prod_{r\in\Pp}\left(\sum_{a\ge0}K(r^a,s)\right)
=\prod_{r\in\Pp}(1-r^{-s})^{-1}
=\zeta(s).
\label{eq:eulerface}
\end{equation}
Unique factorization makes every finite-support choice of nonnegative integer occupations $(a_r)$ correspond to exactly one integer $m$. Absolute convergence for $\Re s>1$ justifies the series/product rearrangement. For real $s=\sigma>1$, the normalized weights $(1-r^{-\sigma})r^{-a\sigma}$ form a geometric probability law; for complex $s$, the displayed formula is an analytic generating identity.

The two faces can therefore be summarized as follows:
\[
\text{Dirichlet face: enumerate complete integer states }m,
\]
\[
\text{Euler face: generate those states from prime-power occupations.}
\]
The common interface is the atom $m^{-s}$.

\section{Relation to probabilistic zeta constructions}
For real $\sigma>1$, the normalized weights
\[
\Pp_\sigma(N=m)=\frac{m^{-\sigma}}{\zeta(\sigma)}
\]
define the classical zeta distribution. For the logarithmic image $Y=-\log N$,
\[
\mathbb E_\sigma e^{itY}=\frac{\zeta(\sigma+it)}{\zeta(\sigma)}.
\]
It is this logarithmic distribution that has the additive infinite-divisibility and compound-Poisson description with prime-power jumps $-a\log r$; see Saito and Tanaka \cite{SaitoTanaka2012}. This statement must be distinguished from additive infinite divisibility of the integer-valued variable $N$. Completed-zeta characteristic-function constructions and RH-related divisibility criteria have been developed by Nakamura \cite{Nakamura2015} and Nakamura--Suzuki \cite{NakamuraSuzuki2023}.

The present paper uses probability at a different entry point. It starts from a Bernoulli likelihood kernel and then analytically extends its exponent split. Once $n$ is replaced by complex $s$, $F(z,s,k)$ is no longer interpreted as a probability mass. Accordingly, no new claim about infinite divisibility or characteristic functions follows merely from the level-set lift. The probabilistic literature is relevant because it confirms that logarithmic prime-power coordinates and Euler-factor decompositions already have precise stochastic realizations; the contribution here is the explicit Bernoulli level representation. Centering $\Re(k/n)=1/2$ refers to the integer-exponent central branch, not arbitrary complex $s$.

\section{Coordinate scope, limitations, and relation to the Riemann Hypothesis}
The DCCQ framework uses three distinct analytic coordinate levels.

\paragraph{Primitive Bernoulli coordinate.}
The binary level and continuation rule construct $z=1/2\pm iu$ directly. The quadratic coordinate $Q(z)=z(1-z)$ then gives a real continuous level with an exact integer quantization.

\paragraph{Normalized split coordinate.}
The central branch $\kappa=k/n$ also lies on $\Re\kappa=1/2$. Its imaginary coordinate is generated by the inverse-integer level equation and is related to, but not identical with, the $u$ coordinate of $z$.

\paragraph{Later Dirichlet exponent.}
The variable $s$ introduced in the complex lift is distinct and unrestricted. The paper therefore does not obtain the Riemann Hypothesis by inserting or deriving a condition $\Re s=1/2$ for zeta zeros. Instead, the critical-line geometry appears earlier in the Bernoulli/split coordinates and is compared with zeta-zero data through the diagnostic coordinate $L_k=1/4+\gamma_k^2$.

The zero-sensitive comparison does not identify individual prime powers with individual zeros. The classical explicit formula supplies a global prime-power/zero coupling, while the conditional Weyl statistics ask a narrower question: whether arithmetic classes of the integer shell condition the residual phase. Exact identities, classical bridges, descriptive numerical controls, and this open statistical question are separate checkable layers.

The logarithmic formulas use the principal branch for $z=1/2+iu$ and $1-z$. Other logarithm branches produce shifted level lattices, so the branch convention is part of the coordinate specification. Likewise, for fixed parameters the scalar coordinate contains no more factorization information than $\log m$; its value is to make explicit, inside one analytic kernel, the projection from independent prime-coordinate occupations to a scalar complex level. The term \emph{superposition} refers to finite-support prime-exponent space, not to a quantum-mechanical state or an asserted spectral operator.

\section{Conclusion}
The DCCQ construction can be read in five aligned layers. First, the starting object is the likelihood of one ordered Bernoulli word, $p^k(1-p)^{n-k}$. The first nonunit inverse-integer level $2^{-n}$ selects $p=1/2$ as the unique real split-independent anchor, and complement-preserving complex continuation generates the conjugation-symmetric vertical line $1/2\pm iu$. The normalized central split coordinate independently inherits the same real part.

Second, the integers that label the construction separate by arithmetic role. The integer $2$ is distinguished by its binary-anchor function; prime powers are repeated occupations of one prime generator; composites with at least two distinct prime divisors are finite multi-generator superpositions. This factorization classification acts on a geometry already obtained from the primitive level problem.

Third, the quadratic coordinate $Q(z)=1/4+u^2$ adds an exact real level geometry with a sharp minimum and explicit inverse. Restriction to integer values gives a discrete arithmetic quantization. Passing supplied critical-line ordinates through $L_k=1/4+\gamma_k^2$ supplies a diagnostic of the coordinate; the first forty-zero prime-hit comparison is descriptive.

Fourth, every zero-induced level decomposes exactly as
\[
L_k=N_k+\delta_k,\qquad
\delta_k=\wtB\!\left(\gamma_k^2+\frac34\right),
\]
so the integer shell and centered periodic-Bernoulli residual are complementary coordinates. Circularization gives $Z_k=e^{2\pi i\delta_k}=i e^{2\pi i\gamma_k^2}$, removing the integer shell and isolating $\gamma_k^2\bmod1$. The local $\wtB$ role is distinct from the higher Bernoulli numbers occurring upstream in Euler--Maclaurin/zeta analysis. The classical explicit formula couples the global von Mangoldt prime-power spectrum to the zero spectrum, while the reported finite correlations are descriptive and do not establish independence. The remaining question is statistical: conditional Weyl moments ask whether arithmetic classes of $N_k$ leave persistent harmonics in $Z_k$.

Fifth, the full complex level-set lift introduces a distinct spectral exponent $s$ and represents each supplied Dirichlet atom $m^{-s}$ inside the same multiplicative architecture. Prime-exponent space retains the full factorization state, while the scalar level coordinate collapses it through $\log m$. The atom then presents the two classical zeta faces in complementary roles: the Dirichlet series enumerates already assembled integers, whereas the Euler product generates them from local prime-power towers.

The resulting DCCQ framework is therefore algebraic and organizational rather than a proof of RH. It creates a precise interface between continuous conjugation-symmetric geometry, discrete arithmetic factorization, Bernoulli residual structure, and the classical prime-power/zero spectrum while keeping exact statements, known theory, finite numerical controls, and open tests distinct.

\appendix
\section{Numerical diagnostics and statistical controls}
\label{app:diagnostics}

\subsection{Exact level decomposition}
For each supplied critical-line ordinate $\gamma_k$, define
\[
L_k=\frac14+\gamma_k^2,\qquad
N_k=\left\lfloor L_k+\frac12\right\rfloor,\qquad
\delta_k=L_k-N_k.
\]
For reproducibility, the source package includes a cache of the first $200$ positive ordinates independently computed with \texttt{mpmath.zetazero} at $40$ decimal-digit working precision, together with code regenerating the retained numerical displays. The cache supplies all numerical values reported here. Odlyzko's tables \cite{OdlyzkoZeros} are a separate reference, not the source of this cache. The cached approximations are not certified zero enclosures. Calculations use the full cached precision; displayed digits are rounded. The first ten levels are shown in Table~\ref{tab:first10}. The factorization and residual are displayed side by side without asserting that either algebraically determines the other.

\begin{table}[H]
\centering
\caption{First ten quadratic level coordinates and nearest-integer arithmetic shells.}
\label{tab:first10}
\small
\begin{tabular}{rrrrll}
\toprule
$k$ & $\gamma_k$ & $L_k$ & $N_k$ & factorization of $N_k$ & $\delta_k$\\
\midrule
1 & 14.134725 & 200.040455 & 200 & $2^3\cdot5^2$ & $+0.040455$\\
2 & 21.022040 & 442.176151 & 442 & $2\cdot13\cdot17$ & $+0.176151$\\
3 & 25.010858 & 625.792997 & 626 & $2\cdot313$ & $-0.207003$\\
4 & 30.424876 & 925.923087 & 926 & $2\cdot463$ & $-0.076913$\\
5 & 32.935062 & 1084.968282 & 1085 & $5\cdot7\cdot31$ & $-0.031718$\\
6 & 37.586178 & 1412.970789 & 1413 & $3^2\cdot157$ & $-0.029211$\\
7 & 40.918719 & 1674.591566 & 1675 & $5^2\cdot67$ & $-0.408434$\\
8 & 43.327073 & 1877.485279 & 1877 & $1877$ & $+0.485279$\\
9 & 48.005151 & 2304.744511 & 2305 & $5\cdot461$ & $-0.255489$\\
10 & 49.773832 & 2477.684400 & 2478 & $2\cdot3\cdot7\cdot59$ & $-0.315600$\\
\bottomrule
\end{tabular}
\end{table}

\subsection{Circular residual geometry}
Figure~\ref{fig:circle} illustrates the identification of the endpoints of the centered residual interval by the circular coordinate.

\begin{figure}[!htbp]
\centering
\includegraphics[width=0.60\linewidth]{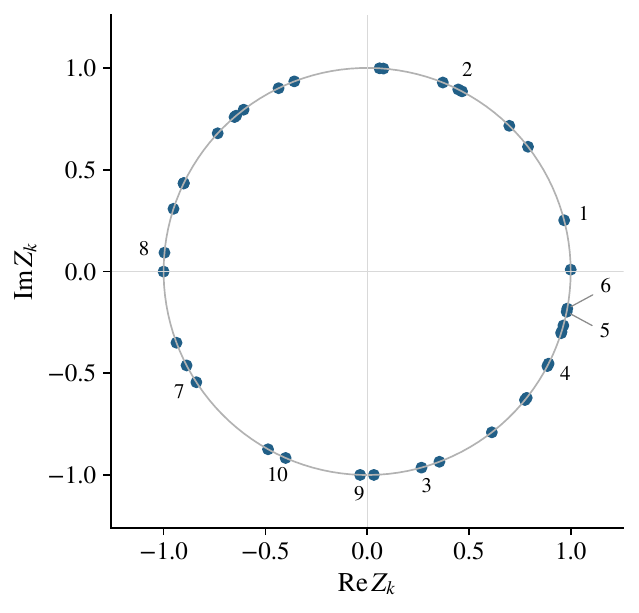}
\caption{Unit-circle representation $Z_k=e^{2\pi i\delta_k}$ for the first forty zeros. Labels 1--10 identify the first ten levels. Circularization removes the artificial discontinuity between residuals near $+1/2$ and $-1/2$.}
\label{fig:circle}
\end{figure}

\subsection{Fourier/Weyl diagnostics and coarse factorization controls}
The Kolmogorov--Smirnov (KS) reference $p$-value below is calculated under independent and identically distributed (iid) uniform sampling. The zero sequence is deterministic, so this is not a calibrated significance test for it. The same limitation applies to the iid-uniform Fourier reference scale. Neither a small finite correlation nor a model reference establishes asymptotic conditional uniformity. The linear correlations of $\delta_k$ also depend on its coordinate cut and are not tests of circular independence.

Figure~\ref{fig:weyl} reports the unconditional magnitudes $|K^{-1}\sum_{k=1}^{K}Z_k^h|$ at $K=200$. It supplies a baseline for the conditional moments in \eqref{eq:Weyl}; no conditional-class Fourier result is asserted.

\begin{figure}[!htbp]
\centering
\includegraphics[width=0.85\linewidth]{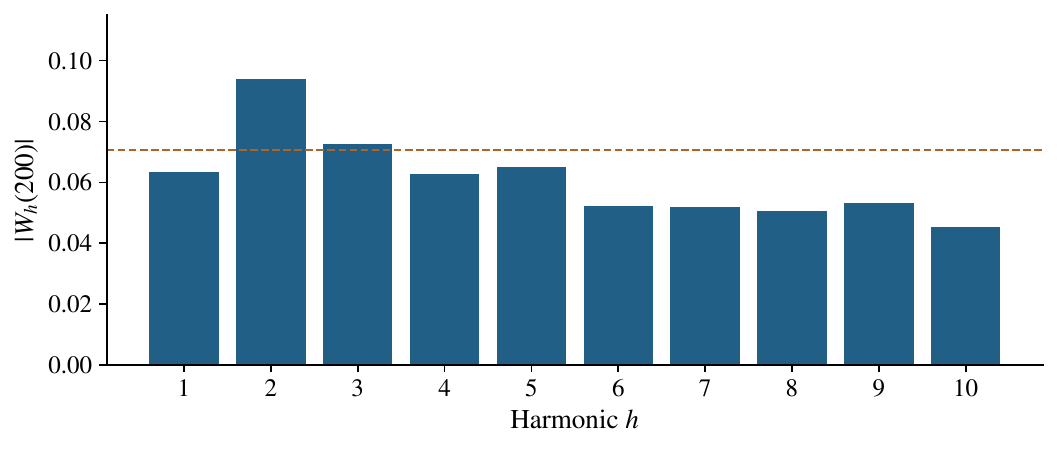}
\caption{First ten unconditional Weyl magnitudes $|200^{-1}\sum_{k=1}^{200}Z_k^h|$. The dashed reference is the root-mean-square magnitude $1/\sqrt{200}$ under an iid-uniform model, not a significance threshold.}
\label{fig:weyl}
\end{figure}

\begin{table}[H]
\centering
\caption{Recomputed first-200 residual diagnostics.}
\label{tab:diagnostics}
\begin{tabular}{lr}
\toprule
Diagnostic & Value\\
\midrule
Sample size & 200\\
Mean of $\delta_k$ & $-0.012976$\\
Variance of $\delta_k$ (divisor $K$) & $0.083552$\\
Uniform benchmark variance & $0.083333$\\
KS statistic vs. $U[-1/2,1/2)$ & $0.0642$\\
KS iid-uniform reference $p$-value & $0.3661$\\
$\operatorname{corr}(\delta,\omega(N))$ & $-0.1106$\\
$\operatorname{corr}(\delta,\Omega(N))$ & $-0.1288$\\
$\operatorname{corr}(|\delta|,\omega(N))$ & $-0.0839$\\
$\operatorname{corr}(|\delta|,\Omega(N))$ & $-0.1074$\\
\bottomrule
\end{tabular}
\end{table}

\subsection{Distinct Bernoulli roles}

\begin{table}[H]
\centering
\caption{Bernoulli objects used in distinct roles.}
\label{tab:bernoulli-roles}
\small
\begin{tabular}{>{\raggedright\arraybackslash}p{0.21\linewidth}>{\raggedright\arraybackslash}p{0.33\linewidth}>{\raggedright\arraybackslash}p{0.36\linewidth}}
\toprule
Object & Exact/classical form & Role in the present framework\\
\midrule
Periodic $\wtB$ & $\delta_k=\wtB(\gamma_k^2+3/4)$ & Exact centered residual readout\\
Bernoulli numbers & Eq.~\eqref{eq:Bn} & Indexed parity/sign and correction hierarchy\\
Euler--Maclaurin & $B_2,B_4,\ldots$ as correction coefficients & Classical upstream link to zeta analysis\\
\bottomrule
\end{tabular}
\end{table}

\subsection{Arithmetic shell classes for conditional Weyl tests}
\begin{table}[H]
\centering
\caption{Candidate arithmetic conditioning classes.}
\label{tab:classes}
\small
\begin{tabular}{>{\raggedright\arraybackslash}p{0.24\linewidth}>{\raggedright\arraybackslash}p{0.30\linewidth}>{\raggedright\arraybackslash}p{0.36\linewidth}}
\toprule
Class $\mathcal A$ & Definition/example & Purpose\\
\midrule
Prime shell & $N_k$ prime & Single-generator arithmetic shell\\
Pure prime power & $N_k=p^a$, $a\ge2$ & One base prime with repeated exponent\\
Semiprime & $\Omega(N_k)=2$ & Minimal two-factor composite class\\
Fixed $\omega$ & $\omega(N_k)=r$ & Controls number of distinct base primes\\
Fixed $\Omega$ & $\Omega(N_k)=r$ & Controls multiplicity-weighted complexity\\
Factorization signature & e.g. exponent partition $(3,2)$ for $2^3 5^2$ & Retains richer generator structure\\
\bottomrule
\end{tabular}
\end{table}

Classes can overlap: a prime square is both a proper prime power and a semiprime. A conditional moment requires a nonempty selected sample, and its asymptotic interpretation requires the selected sample size to diverge.

Among the first forty zero-induced shells, the prime cases occur at $k=8,14,24,32,39$. This is the same $5/40$ count already compared with the PNT heuristic in the main text.

\clearpage
\bibliographystyle{plain}
\bibliography{references}

\end{document}